\documentclass{article}

\usepackage{microtype}
\usepackage{graphicx}
\usepackage{float}  
\usepackage{subcaption}
\usepackage{booktabs} 
\usepackage{amssymb}
\usepackage{pifont}

\usepackage{hyperref}

\newcommand{\xmark}{\ding{55}}

\usepackage[preprint]{icml2026}

\usepackage{amsmath}
\usepackage{amssymb}
\usepackage{mathtools}
\usepackage{amsthm}
\usepackage{bbm}

\usepackage[capitalize,noabbrev]{cleveref}

\theoremstyle{plain}
\newtheorem{theorem}{Theorem}[section]
\newtheorem{proposition}[theorem]{Proposition}
\newtheorem{lemma}[theorem]{Lemma}
\newtheorem{corollary}[theorem]{Corollary}
\theoremstyle{definition}
\newtheorem{definition}[theorem]{Definition}
\newtheorem{assumption}[theorem]{Assumption}
\theoremstyle{remark}
\newtheorem{remark}[theorem]{Remark}

\usepackage[textsize=tiny]{todonotes}

\icmltitlerunning{Step-by-Step Causality: Transparent Causal Discovery with Multi-Agent Tree-Query and Adversarial Confidence Estimation}

\begin{document}

\twocolumn[
  \icmltitle{Step-by-Step Causality: Transparent Causal Discovery with Multi-Agent Tree-Query and Adversarial Confidence Estimation}



  \icmlsetsymbol{equal}{*}

  \begin{icmlauthorlist}
    \icmlauthor{Ziyi Ding}{equal,SIGS}
    \icmlauthor{Chenfei Ye-Hao}{equal,SIGS}
    \icmlauthor{Zheyuan Wang}{sch}

    \icmlauthor{Xiao-Ping Zhang}{SIGS}

  \end{icmlauthorlist}

  \icmlaffiliation{SIGS}{Tsinghua Shenzhen International Graduate School, Tsinghua University, Shenzhen, China}
  \icmlaffiliation{sch}{Zhili College, Tsinghua University, Beijing, China}

  \icmlcorrespondingauthor{Xiao-Ping Zhang}{xpzhang@ieee.org}
  \icmlkeywords{Machine Learning, ICML}

  \vskip 0.3in
]



\printAffiliationsAndNotice{\icmlEqualContribution}

\begin{abstract}
Causal discovery aims to recover ``what causes what'', but classical constraint-based methods (e.g., PC, FCI) suffer from error propagation, and recent LLM-based causal oracles often behave as opaque, confidence-free black boxes. This paper introduces \emph{Tree-Query}, a tree-structured, multi-expert LLM framework that reduces pairwise causal discovery to a short sequence of queries about backdoor paths, (in)dependence, latent confounding, and causal direction, yielding interpretable judgments with robustness-aware confidence scores. Theoretical guarantees are provided for asymptotic identifiability of four pairwise relations. On data-free benchmarks derived from Mooij et al.\ and UCI causal graphs, Tree-Query improves structural metrics over direct LLM baselines, and a diet--weight case study illustrates confounder screening and stable, high-confidence causal conclusions. Tree-Query thus offers a principled way to obtain data-free causal priors from LLMs that can complement downstream data-driven causal discovery. Code is available at \url{https://anonymous.4open.science/r/Repo-9B3E-4F96/}.
\end{abstract}

\section{Introduction}
Causal discovery is a central task in data science: across domains such as medicine, economics, and recommendation, the goal is to understand \emph{what causes what}, not just which variables are correlated. Recovering causal relations from observational data or existing knowledge is therefore key to explaining complex systems and designing effective interventions.

Classical causal discovery algorithms such as \textbf{PC} and \textbf{FCI} incrementally construct a causal graph via long sequences of conditional independence tests.\cite{spirtesCausationPredictionSearch1993,Spirtes1991AnAF}  
This design has several well-known pain points: (i) strong chain dependency, where early independence errors propagate and amplify through later edge deletions and orientations, often corrupting the final graph;\cite{chenCausalStructuralLearning2023,klopotekTooFastCausal2018,stroblEstimatingControllingFalse2017a}  
(ii) limited guidance on \emph{which} intermediate conclusions are trustworthy; and (iii) difficulty integrating domain knowledge or targeted experiments to selectively validate or override specific decisions.

The emergence of large language models (LLMs) suggests a new way to reason about causality. A growing line of work directly treats an LLM as a causal oracle: the model is asked questions such as ``Does $X$ cause $Y$?'' or ``Is there a confounder between $X$ and $Y$?'', and its answers are used to construct causal graphs.\cite{zhangCausalGraphDiscovery2024,zannaUncoveringBiasPaths2025,longCanLargeLanguage2024,longCausalDiscoveryLanguage2023,leMultiAgentCausalDiscovery2025,jiralerspongEfficientCausalGraph2024,darvariuLargeLanguageModels2024,banIntegratingLargeLanguage2025}  
Other works embed LLMs into traditional pipelines, for example as semantic conditional independence testers inside PC-like procedures.\cite{kadziolkaCausalReasoningPieces2025,cohrsLargeLanguageModels2024}  
These approaches leverage the rich world knowledge and reasoning ability of LLMs, but introduce new pain points: outputs are prompt-sensitive and sometimes unstable, there is rarely a calibrated notion of \emph{confidence} in each causal claim, chain error accumulation from PC-like controllers persists, and latent confounders remain hard to identify explicitly. Overall, the internal reasoning remains largely opaque, limiting the credibility and interpretability of the resulting causal graphs.

To address these issues, we propose \emph{Tree-Query}, a novel tree-query-based multi-expert framework that turns LLM-based causal discovery into a transparent, stepwise, and confidence-aware reasoning process. Given a variable set $V$ and any target pair $(X_1, X_2)$, the controller runs a fixed sequence of small causal queries (backdoor existence, conditional independence, latent confounding, causal direction). Each query is answered by a panel of specialized LLM ``experts'', challenged by an Adversarial Confidence Estimator that assigns a robustness-based confidence score, and the final causal graph is assembled from these locally trusted pairwise relations.

Compared to traditional methods and existing LLM-embedded approaches, Tree-Query offers several advantages:
(i) \textbf{Transparent and interpretable}: causal discovery is decomposed into simple, standardized questions, with all queries, experts, votes, and outcomes recorded in a unified reasoning log;
(ii) \textbf{Confidence-aware}: adversarial confidence estimation explicitly tests the stability of key assertions and allows the controller to prioritize high-confidence branches;
(iii) \textbf{Latent confounding aware}: the framework can directly query latent confounders between two variables and route them to appropriate experts (graph-based, counterfactual, or domain-knowledge);
(iv) \textbf{Error-isolating}: each variable pair is processed by its own Tree-Query pipeline without shared conditioning sets, so errors remain local instead of propagating across the graph.

Overall, Tree-Query casts causal discovery as a modular pipeline with clear inputs (variable sets and pairs), explicit intermediate judgments, confidence-scored outputs, and a globally consistent causal graph built from these components.

\begin{figure*}[ht]
  \vskip 0.1in
  \begin{center}
    \centerline{\includegraphics[width=0.9\textwidth]{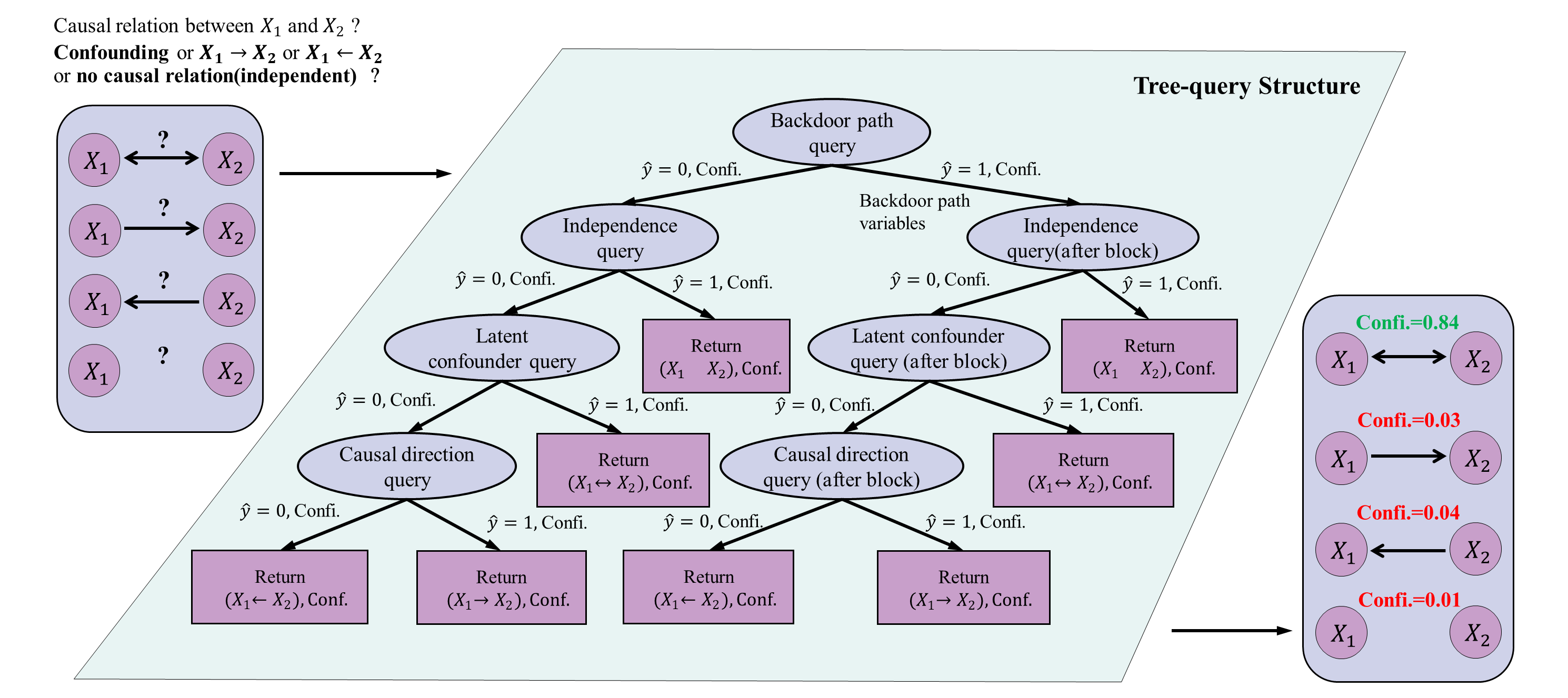}}
\caption{Tree-Query structure for inferring the causal relation between a variable pair $(X_1,X_2)$. 
Starting from an unknown relation (left), Tree-Query evaluates a fixed sequence of queries on the Tree-Query plane. 
The root issues a \texttt{backdoor\_path} query: if no backdoor path is detected ($\hat{y}=0$), the procedure follows the left branch and successively asks \texttt{independence}, \texttt{latent\_confounder}, and, if needed, \texttt{causal\_direction}; if a backdoor path is detected ($\hat{y}=1$), the corresponding variables are conceptually blocked and the same sequence of queries is executed along the ``after block'' branch on the right. 
Each internal node outputs a local binary decision $\hat{y}$ with an associated confidence (``Confi.''), and each root-to-leaf path ends in one of the candidate relations $\{X_1 \perp X_2,\; X_1 \leftrightarrow X_2,\; X_1 \rightarrow X_2,\; X_2 \rightarrow X_1\}$ shown in purple. 
Algorithm~\ref{alg:tree_checks} evaluates all queries in this structure, and Algorithm~\ref{alg:tree_query_overall} aggregates the leaf outcomes into a final relation and overall confidence (right). 
The implementation of each query node by the multi-expert and confidence-estimation modules is given in Fig.~\ref{MES_ACE}.}

    \label{TreeQuery}
  \end{center}
\end{figure*}

\begin{figure*}[ht]
  \vskip 0.1in
  \begin{center}
    \centerline{\includegraphics[width=0.8\textwidth]{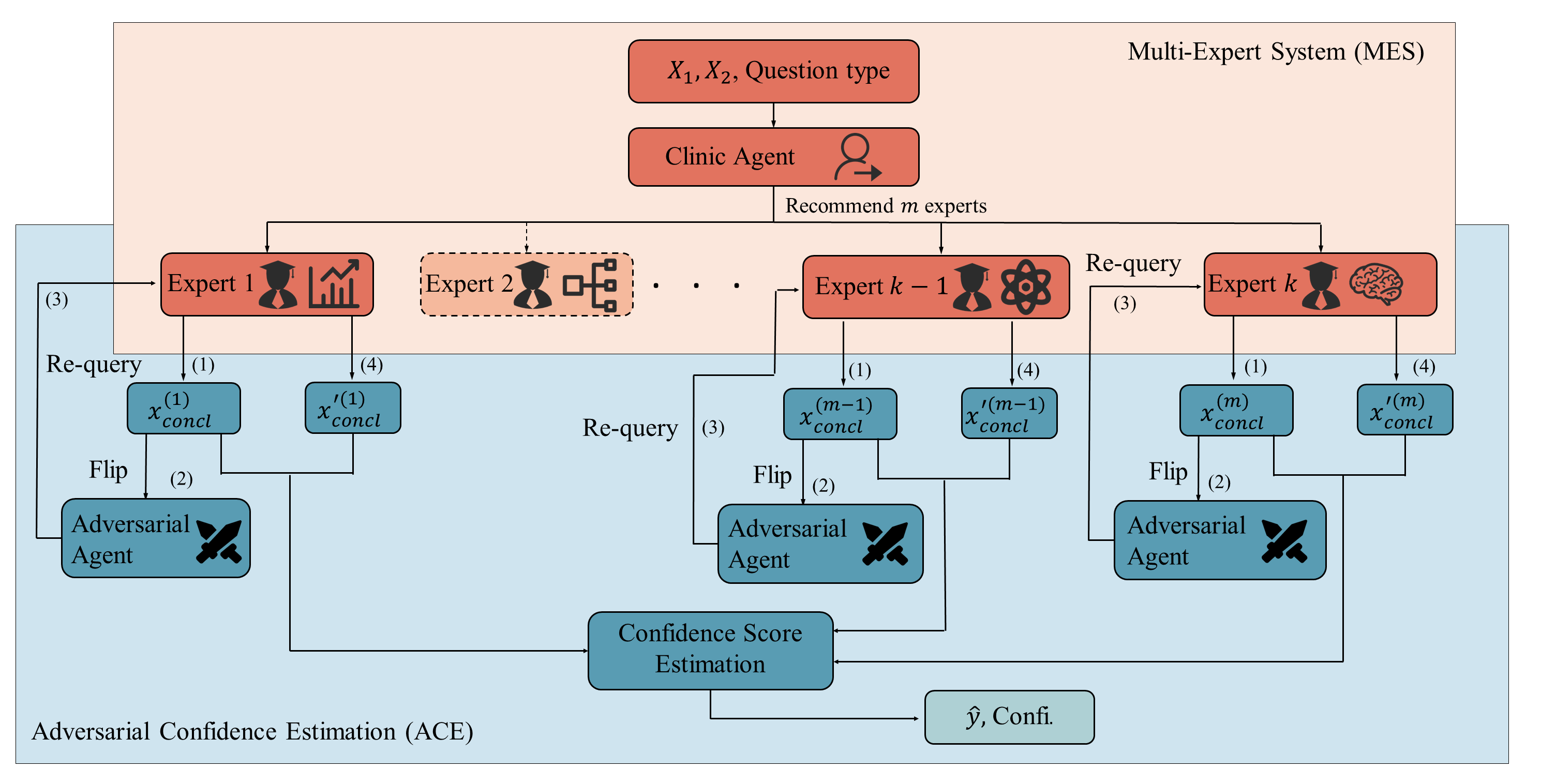}}
\caption{Multi-Expert System (MES) with Adversarial Confidence Estimation (ACE) for a single Tree-Query node. 
Given a variable pair $(X_1,X_2)$ and a query type, the clinic agent selects $m$ experts from $K$ candidates and forwards the query. 
Each selected expert returns a binary conclusion $x^{(i)}_{\text{concl}}$ (step (1)). 
An adversarial agent then flips the conclusion and generates an adversarial argument, which is used to re-query the same expert configuration (steps (2)--(3)), producing perturbed conclusions ${x'}^{(i)}_{\text{concl}}$ (step (4)). 
ACE aggregates the original and perturbed conclusions across experts and adversarial personas to obtain a majority-vote label $\hat{y}$ and a robustness-based confidence score (``Confi.''). 
This MES--ACE module implements each query node in the Tree-Query structure shown in Fig.~\ref{TreeQuery}; 
the corresponding pseudocode for the combined MES-ACE module is given in Algorithm~\ref{alg:mes_ace_combined} in Appendix~\ref{appendix:mes_ace}.}

    \label{MES_ACE}
  \end{center}
\end{figure*}

\section{Problem Setup}
In this section, we define the notation and underlying assumptions, and present a formulation of the problem being investigated.

\subsection{Notation}
Let $X = \{X_1, X_2, \dots, X_d\}$ denote the set of observed variables. We assume there exists a true underlying causal graph $G^*$ over $X$, where edges may be directed ($\rightarrow$) or bidirected ($\leftrightarrow$) to represent latent confounding. Let $\mathcal{G}$ denote the space of all such partially directed acyclic graphs (PDAGs). The joint probability distribution over $X$ is denoted by $P(X)$. We access an oracle—an ensemble of LLM experts—that provides noisy but imformative judgments on causal relations, where the error probability of each expert is bounded by a parameter $\alpha$.

The following assumptions ensure the existence of true causal graph $G^*$ and $P(X)$.  

\subsection{Assumptions}
We consider causal discovery under a structured set of assumptions combining classical causal principles with the practical reliability limits of LLM-based reasoning.

\begin{assumption}[Causal Markov Condition]
\label{a1}
The distribution $P$ factorizes according to the d-separation relationships in causal graph $G$:
\(
X_1 \perp X_2 \mid X_3 \text{ in } P \iff X_1 \text{ and } X_2 \text{ are d-separated by } X_3 \text{ in } G.
\)
where $X_i (i=1,2,3)$ are causal variables.
\end{assumption}

\begin{assumption}[Faithfulness Condition]
\label{a2}
All conditional independences in $P$ arise from d-separation in $G$, with no coincidental independences.
\end{assumption}

Assumptions~\ref{a1}--\ref{a2} are classical principles.

\begin{assumption}[Weak Causal Sufficiency]
\label{a3}
Latent confounders may exist but can be detected through observable independence patterns and represented as bidirected edges.
\end{assumption}
As the latent confounders are difficult to be causally identified, classical causal discovery algorithms often assume the non-existence of latent confounders, which is rarely true in real-world situations. We assume that latent confounders leave statistical traces that can be keenly captured by detective methods like LLMs, thereby allowing us to mark these relations as confounded rather than regarding them as completely indeterminable.

\begin{assumption}[Bounded Reliability]
\label{a4}
Each LLM expert has error probability at most $\alpha < 0.5$ in independence and direction judgments; errors are approximately independent, enabling ensemble aggregation to reduce total inference error.
\end{assumption}
This assumption allows LLM experts make mistakes. Strictly requiring $\alpha < 0.5$ means the experts work better than random guessing. Conditioned on the fixed input pair$(X_1,X_2)$, the errors made by experts are attributed to their internal machanisms, rendering them mutually independent.

Assumptions~\ref{a3}--\ref{a4} reflect that LLMs can semantically infer hidden relations but imperfectly, motivating a probabilistic reliability bound that ensures ensemble-level robustness.

\subsection{Problem Formulation}

Given a set of variables $X_1,\dots,X_d$, the ground-truth causal relation for each ordered pair $(X_i, X_j)$ with $i \neq j$ is one of
\[
R \in \{ X_i \rightarrow X_j,\; X_j \rightarrow X_i,\; X_i \leftrightarrow X_j,\; X_i \perp X_j \}.
\]
Instead of observing samples from $P(X)$, we can query an ensemble of LLM experts that provide noisy but informative binary judgments about high-level causal properties (e.g., (in)dependence, backdoor paths, latent confounding, or causal direction). Our goal is to design a querying and aggregation procedure that, for each pair $(X_i, X_j)$, uses only a finite number of such queries to output an estimated relation (with a confidence score). Collecting all pairwise decisions yields an estimated graph $\hat G$, which we evaluate by the number of correctly recovered edges or, equivalently, the Structural Hamming Distance (SHD) to the true graph $G^\star$.

\section{Methodology}

\subsection{Tree-Query Framework}
\label{sec:tree_query}

Tree-Query reduces pairwise causal discovery to a fixed sequence of simple queries. For any pair $(X_1, X_2)$ and variable set $V$, it runs four query types through a Multi-Expert System (MES) and an Adversarial Confidence Estimator (ACE), and then applies a deterministic decision rule to output one of
\(
\hat{R} \in \{X_1 \perp X_2,\; X_1 \leftrightarrow X_2,\; X_1 \rightarrow X_2,\; X_2 \rightarrow X_1\}
\)
together with a confidence score.

\paragraph{Design principles.}
Tree-Query follows two principles: (i) reduce global causal discovery to a small set of human-understandable questions, and (ii) order these questions so that reliable, easy-to-verify information is used first. Accordingly, the four query types are arranged from coarse to fine: \texttt{backdoor\_path} checks high-level structural connectivity; \texttt{independence} serves as an early-stopping test between ``no relation'' and ``potential relation''; \texttt{latent\_confounder} separates spurious dependence from genuine causal links; and \texttt{causal\_direction} is used only when a direct relation remains to be oriented. Each query corresponds to a concrete, testable claim (e.g., ``independent after adjustment''). The use of a fixed, finite sequence bounds query depth and is crucial for the asymptotic identifiability result in Theorem~\ref{thm:asym_identifiability}.

\paragraph{Query types and MES--ACE interface.}
Tree-Query employs four primitive query types: \texttt{backdoor\_path}, \texttt{independence}, \texttt{latent\_confounder}, and \texttt{causal\_direction}. Each type is implemented by the combined MES-ACE module (Algorithm~\ref{alg:mes_ace_combined} in Appendix~\ref{appendix:mes_ace}). For any query $Q$, the MES-ACE module outputs a triple
\[
(\hat{y}, c, \mathcal{R}) = \textsc{MES-ACE}(Q, X_1, X_2, V; K, m, N)
\]
where $\hat{y} \in \{0,1\}$ is the majority-voted binary label, $c \in [0,1]$ is the adversarially-tested confidence score, and $\mathcal{R}$ records the expert log. The binary label $\hat{y}$ is deterministically mapped to a relation-specific label $\textit{rel}$ depending on $Q$ (e.g., for \texttt{independence}, $\hat{y}=1$ means ``independent''). We use $(\textit{rel}, c, \mathcal{R})$ as the interface for a query of type $Q$, omitting $(K,m,N)$ for brevity.

\begin{algorithm}[t]
\caption{Tree-Query (overall controller)}
\label{alg:tree_query_overall}
\begin{algorithmic}[1]
  \STATE \textbf{Input:} variables $(X_1, X_2)$, variable set $V$, 
         MES-ACE module parameters, confidence threshold $\tau$
  \STATE \textbf{Output:} final relation $\hat{R}$, confidence $\hat{c}$
  \STATE $(\textit{have\_backdoor}, c_{\mathrm{bd}}, \textit{results}) 
         \gets \textsc{Tree-Query-Checks}(X_1, X_2, V)$
  \STATE $(\hat{R}, \hat{c}) \gets 
         \textsc{DecisionRule}(\textit{have\_backdoor}, c_{\mathrm{bd}}, \textit{results}, \tau)$
  \STATE \textbf{return} $(\hat{R}, \hat{c})$
\end{algorithmic}
\end{algorithm}

\paragraph{Overall Tree-Query procedure.}
The Tree-Query procedure separates \emph{evidence collection} from \emph{decision aggregation}. Evidence collection is handled by \textsc{Tree-Query-Checks} (Algorithm~\ref{alg:tree_checks}), which first tests for the existence of any backdoor path between $(X_1, X_2)$ and then, for each conceptual branch (``after\_block'' and ``no\_backdoor''), evaluates independence, latent confounding, and causal direction via the combined MES-ACE module (Algorithm~\ref{alg:mes_ace_combined} in Appendix~\ref{appendix:mes_ace}). The resulting set of branch-specific outputs is then passed to a deterministic decision rule that aggregates them into a final causal relation, as in Algorithm~\ref{alg:tree_query_overall}.

\begin{algorithm}[h]
\caption{Tree-Query-Checks}
\label{alg:tree_checks}
\begin{algorithmic}[1]
  \STATE \textbf{Input:} variables $(X_1, X_2)$, variable set $V$
  \STATE \textbf{Output:} $\{\textit{have\_backdoor}, c_{\mathrm{bd}}, \textit{results}\}$
  \STATE $S_Q \gets \{\texttt{independence},\ \texttt{latent\_confounder},$ 
  \STATE $\texttt{causal\_direction}\}$
  \STATE $(\hat{y}_{\mathrm{bd}}, c_{\mathrm{bd}}, \mathcal{R}_{\mathrm{bd}}) 
         \gets \textsc{MES-ACE}(\texttt{backdoor\_path}, X_1, X_2, V)$
  \STATE $\textit{have\_backdoor} \gets (\hat{y}_{\mathrm{bd}} = 1)$
  \STATE $\textit{results} \gets \emptyset$
  \FOR{each branch $b \in \{\text{``after\_block''}, \text{``no\_backdoor''}\}$}
    \FOR{each query type $Q \in S_Q$}
      \STATE $(\hat{y}, c, \mathcal{R}) \gets \textsc{MES-ACE}(Q, X_1, X_2, V)$
      \STATE $\textit{rel} \gets \textsc{MapToRelation}(Q, \hat{y})$
      \STATE $\textit{results} \gets \textit{results} \cup \{(b, \textit{rel}, c)\}$
    \ENDFOR
  \ENDFOR
  \STATE \textbf{return} $\{\textit{have\_backdoor}, c_{\mathrm{bd}}, \textit{results}\}$
\end{algorithmic}
\end{algorithm}

\paragraph{Decision rule (aggregation layer).}
Given the outputs of \textsc{Tree-Query-Checks}, the decision rule (Algorithm~\ref{alg:decision_rule} in Appendix~\ref{DecisionRule1}) aggregates all branch-specific results in an early-stopping fashion. It first uses the backdoor confidence to decide whether to emphasize the ``after\_block'' branch, the ``no\_backdoor'' branch, or both. Within each branch, it proceeds in three stages:
(i) if strong evidence of independence is found, the branch concludes ``independent''; 
(ii) else, if strong evidence of latent confounding is found, it concludes $X_1 \leftrightarrow X_2$; 
(iii) otherwise, directional evidence (including negative statements such as ``not $\rightarrow$'' / ``not $\leftarrow$'') is converted into a preference between $X_1 \rightarrow X_2$ and $X_2 \rightarrow X_1$. The final output $(\hat{R}, \hat{c})$ is the relation with the highest aggregated confidence across branches.

\paragraph{Comparison and summary.}
Compared with traditional frameworks such as the PC algorithm, which rely on recursive conditional independence tests and can suffer from error propagation through changing conditioning sets, Tree-Query uses a fixed query structure for each causal pair. This decoupling of queries supports additive error control instead of multiplicative error cascades. Overall, Tree-Query follows a clear ``coarse-to-fine'' and ``query-first, aggregate-later'' logic: a stable query tree collects interpretable local evidence, and a deterministic rule turns it into global causal relations, providing a clean basis for the asymptotic identifiability guarantees. An overview of the resulting Tree-Query structure and its MES--ACE instantiation is shown in Fig.~\ref{TreeQuery} and Fig.~\ref{MES_ACE}.

\subsection{Multi-Expert System with Adversarial Confidence Estimation}

\subsubsection{Multi-Expert System for Robust Causal Judgment}
To improve robustness and generalization in complex causal reasoning, this work adopts a \textbf{Multi-Expert System} (MES), inspired by mixture-of-experts (MoE) architectures \citep{jacobs1991adaptive,jordan1994hierarchical,shazeer2017outrageously}. Instead of relying on a single monolithic reasoner, MES maintains a pool of experts, each encoding a distinct causal perspective or reasoning style. During inference, a query of type $Q$ (e.g., one of the four Tree-Query primitives) is routed to a suitable subset of experts via a routing function \texttt{RoutingRules} (detailed in Algorithm~\ref{alg:mes_ace_combined} in Appendix~\ref{appendix:mes_ace}), which selects experts whose specialties match $Q$ (e.g., graph-structure experts for \texttt{backdoor\_path}, statistical experts for \texttt{independence}, domain experts for \texttt{latent\_confounder}). The selected experts issue independent binary judgments under a shared prompt template, and MES aggregates these judgments by majority vote. The module outputs only the final binary label and the expert log $\mathcal{R}$; the confidence score used by Tree-Query is computed downstream by the adversarial confidence estimator.

Table~\ref{tab:experts} in Appendix~\ref{expert_prompt} presents an illustrative subset of experts that can be included in the system. The list is not fixed: experts can be added, removed, or customized depending on task requirements. Each expert corresponds to a coherent causal reasoning angle—ranging from structural causal modeling to intervention analysis or graphical heuristics—and all experts follow unified prompt templates with consistent input formatting and output conventions. The complete prompt templates and example I/O pairs are provided in Appendix~\ref{expert_prompt}.

\subsubsection{Adversarial Confidence Estimator}
\label{sec:ace}

While the Multi-Expert System provides robust binary judgments, it does not inherently offer a calibrated measure of certainty. To address this limitation and enable effective aggregation within the Tree-Query framework, we introduce an \textbf{Adversarial Confidence Estimator} (ACE). Inspired by adversarial training in robust machine learning \citep{madry2019deeplearningmodelsresistant}, ACE explicitly tests the stability of an MES decision under targeted counter-arguments. This process yields a scalar confidence score $c \in [0,1]$, which is subsequently used to weigh and compare assertions across the reasoning tree.

Given a natural-language query $x_{\textnormal{q}}$ (corresponding to some $(Q,X_1,X_2,V)$), MES is first run $N$ times with a fixed expert set, producing answers $\{x_{\textnormal{ans}}^{(i)}\}_{i=1}^{N}$ and binary conclusions $\{x_{\textnormal{concl}}^{(i)}\}_{i=1}^{N} \subset \{\textnormal{``Yes''}, \textnormal{``No''}\}$. Let $N_{\textnormal{yes}}$ and $N_{\textnormal{no}}$ be the counts of ``Yes'' and ``No'', and define the raw majority proportion
\begin{align}
p_0^{\textnormal{raw}} &= \frac{\max(N_{\textnormal{yes}}, N_{\textnormal{no}})}{N} \in [0.5,1], \\
p_0 &= 2p_0^{\textnormal{raw}} - 1 \in [0,1],
\end{align}
with majority label
\begin{equation}
y^* = \arg\max_{y \in \{\textnormal{``Yes''}, \textnormal{``No''}\}} |\{i : x_{\textnormal{concl}}^{(i)} = y\}|.
\end{equation}
To probe robustness, ACE then generates $n$ adversarial agents that argue for the opposite conclusion and re-queries MES under these perturbations, yielding adversarially influenced conclusions $\{{x'}_{\textnormal{concl}, j}^{(i)}\}_{i=1}^{N}$ for $j=1,\dots,n$. In implementation, $n=3$ personas are instantiated as diverse argumentation styles (Contrarian, Deceiver, and Hater); detailed prompt templates are provided in Appendix~\ref{app:ace_prompts}. For each adversarial persona $j$, the \emph{majority-aligned consistency} is defined as
\begin{equation}
\label{consistency_prob}
p_j = \frac{1}{N} \sum_{i=1}^{N} \mathbbm{1}\bigl({x'}_{\textnormal{concl}, j}^{(i)} = y^* \bigr),
\end{equation}
measuring how often the adversarially perturbed answers remain aligned with the original majority conclusion.

Let $\lambda_1,\dots,\lambda_n \ge 0$ with $\sum_{j=1}^n \lambda_j = 1$ denote persona weights (set to $\lambda_j = 1/n$ in all experiments). The final confidence score combines the strength of the original consensus and its robustness to adversarial perturbations:
\begin{equation}
\label{eq:confidence_score}
c \;=\; p_0 \cdot \left(1 - \sum_{j=1}^{n} \lambda_j \cdot \frac{|p_j - p_0^{\textnormal{raw}}|}{p_0^{\textnormal{raw}}}\right),
\end{equation}
where large deviations $|p_j - p_0^{\textnormal{raw}}|$ indicate that the baseline consensus is easily undermined. The complete algorithmic implementation of both MES and ACE is provided in Algorithm~\ref{alg:mes_ace_combined} in Appendix~\ref{appendix:mes_ace}.

\section{Theoretical Analysis}

\subsection{Theoretical Analysis of the Tree-Query Framework}

\begin{theorem}[Asymptotic Identifiability of Tree-Query]
\label{thm:asym_identifiability}
Under Assumptions~\ref{a1}--\ref{a4}, consider any variable pair $(X_1, X_2)$ with its true causal relation
\[
R \in \{\,X_1 \rightarrow X_2,\; X_2 \rightarrow X_1,\; X_1 \leftrightarrow X_2,\; X_1 \perp X_2\,\}
\]
Let Tree-Query consist of $M$ deterministic decision queries (decision stages), and assume that each query is answered by the Multi-Expert System using $m$ independent experts (as in Algorithm~\ref{alg:mes_ace_combined} in Appendix~\ref{appendix:mes_ace}), whose individual error probability satisfies $\alpha < 0.5$.

Then for every relation $R$, the probability that Tree-Query identifies true relation $R$ satisfies
\begin{equation}
P_{\textnormal{correct}}(R) \;\ge\; 1 - M \cdot \exp\!\left[-2m(0.5-\alpha)^2\right].
\end{equation}
Moreover,
\begin{equation}
\lim_{m\to\infty} P_{\textnormal{correct}}(R) = 1, 
\qquad
\lim_{\alpha\to0} P_{\textnormal{correct}}(R) = 1.
\end{equation}
\end{theorem}

\textbf{Proof.} See Appendix~\ref{appendix:proof_asym_identifiability}.

Theorem~\ref{thm:asym_identifiability} shows that, under Assumptions~\ref{a1}--\ref{a4} and $\alpha < 0.5$, Tree-Query is \emph{asymptotically identifiable}. For any of the four possible relations
\[
R \in \{X_1 \rightarrow X_2,\; X_2 \rightarrow X_1,\; X_1 \leftrightarrow X_2,\; X_1 \perp X_2\}
\]
the probability that Tree-Query outputs the correct $R$ can be made arbitrarily close to $1$ by increasing the number of experts $m$ per query, while keeping the tree structure (i.e., the finite number of decision stages $M$) fixed. In other words, the tree-shaped query flow is expressive enough to identify all four relations, and the only limitation comes from expert noise: as long as individual experts are better than random guessing ($\alpha < 0.5$), aggregating sufficiently many of them at each node drives the overall error probability to zero, yielding arbitrarily reliable recovery of the true causal relation.

\begin{proposition}[Overall Causal Graph Identification Reliability]
\label{thm1}
Under Assumptions~\ref{a1}--\ref{a4}, for a causal graph containing $e$ edges, let $P_{\textnormal{pair}}$ denote the probability that Tree-Query correctly identifies the causal relation for a single edge. Then
\begin{equation}
E_{\textnormal{TQ}} \;=\; e \cdot P_{\textnormal{pair}} 
\;\geq\; e \cdot \left[1 - M \cdot \exp\!\left(-2m(0.5-\alpha)^2\right)\right].
\end{equation}
\end{proposition}

\textbf{Proof.} See Appendix~\ref{appendix:Prop1}.

\begin{proposition}[Practical Feasibility Boundary]
\label{prop2}
Under Assumptions~\ref{a1}--\ref{a4},to achieve a target correctness probability $P_{\textnormal{target}}$ at the \emph{pairwise} level, the required number of experts $m$ per query satisfies:
\begin{equation}
m \;\geq\; 
\left\lceil 
\frac{-\ln\!\left(\dfrac{1-P_{\textnormal{target}}}{M}\right)}{2(0.5-\alpha)^2} 
\right\rceil.
\end{equation}
\end{proposition}

\textbf{Proof.} See Appendix~\ref{appendix:Prop2}.

Propositions~\ref{thm1} and \ref{prop2} characterize both the reliability and resource requirements of Tree-Query in graph-based causal discovery. Proposition~\ref{thm1} lower-bounds the expected number of correctly identified edges in a graph of size $e$ in terms of the tree depth $M$, the number of experts $m$ per query, and the individual error rate $\alpha$, while Proposition~\ref{prop2} inverts this relationship to give the minimal $m$ needed to reach a target pairwise correctness level $P_{\textnormal{target}}$, making explicit the trade-off between expert quality, tree complexity, and desired reliability.

\subsection{ACE: Evaluating Distributional Robustness against Adversarial Perturbations}

Let $\mathbf{H}_{\textnormal{q}}$, $\mathbf{H}_{\textnormal{ans}}$, and $\mathbf{H}_{\textnormal{adv}}$ denote semantic representations of the query, the original reasoning, and the adversary-influenced reasoning, respectively. With pretrained matrices $\mathbf{W}^\mathbf{Q}, \mathbf{W}^\mathbf{K}, \mathbf{W}^\mathbf{V}$ and key/query dimension $d$, define the (unnormalized) attention score of a context $\mathbf{H}'$ as
\begin{gather*}
S(\mathbf{H}') := \frac{1}{\sqrt{d}}\, \mathbf{W}^\mathbf{Q}(\mathbf{H}'\mathbf{W}^\mathbf{K})^\top \\
S_{\textnormal{ans}} := S(\mathbf{H}_{\textnormal{ans}}), \quad
S_{\textnormal{adv}} := S(\mathbf{H}_{\textnormal{adv}})
\end{gather*}

\begin{assumption}[Argument strength via spectral alignment]
The score $S(\mathbf{H}')$ monotonically reflects the semantic strength of $\mathbf{H}'$ for the current query. In particular, robust conclusions satisfy $S_{\textnormal{ans}} \ge S_{\textnormal{adv}}$, while a successful adversarial prompt can increase $S_{\textnormal{adv}}$ relative to $S_{\textnormal{ans}}$.
\end{assumption}

We consider binary outputs $\mathcal{Y} = \{\textnormal{``Yes''}, \textnormal{``No''}\}$. The attention aggregation yields a context vector
\begin{equation}
  \mathbf{c} = \frac{1}{Z}\left(e^{S_{\textnormal{ans}}} \mathbf{H}_{\textnormal{ans}}\mathbf{W}^\mathbf{V} + e^{S_{\textnormal{adv}}} \mathbf{H}_{\textnormal{adv}}\mathbf{W}^\mathbf{V}\right),
\end{equation}
where $Z$ is a normalization factor. Let $v_{\textnormal{yes}} = \mathbf{c}^\top \mathbf{w}_{\textnormal{yes}}$ and $v_{\textnormal{no}} = \mathbf{c}^\top \mathbf{w}_{\textnormal{no}}$ be the logits for the two labels, and define
\begin{equation}
  	heta := P(y = \textnormal{``Yes''}) = \frac{e^{v_{\textnormal{yes}}}}{e^{v_{\textnormal{yes}}} + e^{v_{\textnormal{no}}}}.
\end{equation}
Thus the output distribution $\mu$ on $\mathcal{Y}$ is Bernoulli: $\mu \sim \mathrm{Bern}(\theta)$.

\begin{lemma}[Distributional shift]
\label{lemma:distributional_shift}
Let $\mu_{\textnormal{ans}}$ and $\mu_{\textnormal{adv}}$ denote the (population) output distributions over $\mathcal{Y}$ conditioned on the original and adversarial contexts, respectively. Adversarial prompting induces a shift in the Bernoulli parameter $\theta$, redistributing probability mass from the original conclusion $y_{\textnormal{concl}} \in \mathcal{Y}$ to the opposite outcome $\bar{y}_{\textnormal{concl}} = \mathcal{Y} \setminus \{y_{\textnormal{concl}}\}$.
\end{lemma}

Equipping $\mathcal{Y}$ with the discrete metric $d(y,y') = \mathbbm{1}(y \neq y')$, the Wasserstein-1 distance between two empirical measures on $\mathcal{Y}$ reduces to the absolute difference in mass assigned to one label. This yields the following interpretation of the deviation term in Eq.~\eqref{eq:confidence_score}.

\begin{proposition}[Confidence as robustness against distributional shift]
\label{prop:confidence_robustness}
Let $\hat{\mu}_{\textnormal{ans}}$ and $\hat{\mu}_{\textnormal{adv}, j}$ be the empirical measures on $\mathcal{Y}$ corresponding to the original answers and the answers under the $j$-th adversarial agent, respectively. If $\mathcal{Y}$ is equipped with the discrete metric $d(y, y') = \mathbbm{1}(y \neq y')$, then
\begin{equation}
  |p_j - p_0^{\textnormal{raw}}| \;=\; W_1(\hat{\mu}_{\textnormal{ans}}, \hat{\mu}_{\textnormal{adv}, j})
\end{equation}
where $W_1$ is the Wasserstein-1 distance. Consequently, the confidence score in Eq.~\eqref{eq:confidence_score} measures the robustness of the original output distribution to adversarially induced distributional shifts.
\end{proposition}

\textbf{Proof.} See Appendix~\ref{proof:confi}.

Lemma~\ref{lemma:distributional_shift} and Proposition~\ref{prop:confidence_robustness} provide a compact robustness interpretation of ACE. Lemma~\ref{lemma:distributional_shift} formalizes adversarial prompting as a shift in the binary output distribution, while Proposition~\ref{prop:confidence_robustness} shows that the deviation term $|p_j-p_0^{\textnormal{raw}}|$ equals a Wasserstein-1 distance under the discrete metric. Consequently, Eq.~\eqref{eq:confidence_score} can be viewed as rewarding both strong baseline consensus and stability of that consensus under adversarially induced distributional shifts.

\section{Experiment}

\subsection{Benchmark}

All methods are evaluated on the causal structures from
Mooij et al.~\cite{JMLR:v17:14-518} and the UCI Machine Learning Repository~\cite{UCIRepository}, using only the \textbf{ground-truth graphs} and no observational data, so that evaluation is purely \emph{data-free} and focuses on structural reasoning.

\subsubsection{Standard Benchmark}

The \textbf{Standard} benchmark consists of the original directed graphs released by Mooij et al.  
Since data are not used, the task reduces to recovering the directed edges and orientations in these ground-truth structures, isolating the structural reasoning ability of each method.

\subsubsection{Latent-Confounder Benchmark}

To assess \textbf{latent confounding}, a modified benchmark is built by masking each unobserved confounder and inserting a \textbf{bidirected edge} ($A \leftrightarrow B$) between its observed children, following standard PAG semantics.  
Figure~\ref{benchmark} illustrates this transformation, yielding a benchmark that explicitly tests the handling of latent common causes and edge ambiguity.

\subsection{Metrics}

Normalized Discounted Cumulative Gain (NDCG) and Structural Hamming Distance (SHD) are used in the experiment as the metrics. For methods that produce confidence scores over candidate relations, candidates are ranked by confidence and evaluated with
\(
\mathrm{DCG}@K = \sum_{i=1}^{K} \frac{\mathrm{rel}_i}{\log_2(i+1)}, \quad 
\mathrm{NDCG}@K = \frac{\mathrm{DCG}@K}{\mathrm{IDCG}@K} \in [0,1],
\)
where $\mathrm{rel}_i \in \{0,1\}$ and $\mathrm{IDCG}@K$ is the ideal DCG. NDCG is reported only for methods that output a ranked list (e.g., Tree-Query); Direct LLM returns a single label per pair, so NDCG is marked as N/A. SHD counts the number of edge additions, deletions, and orientation reversals needed to turn the predicted graph into the ground truth, computed on directed graphs for the Standard benchmark and on bidirected-augmented graphs for the Latent benchmark.

\begin{figure}[t]
  \centering
  \includegraphics[width=1\linewidth]{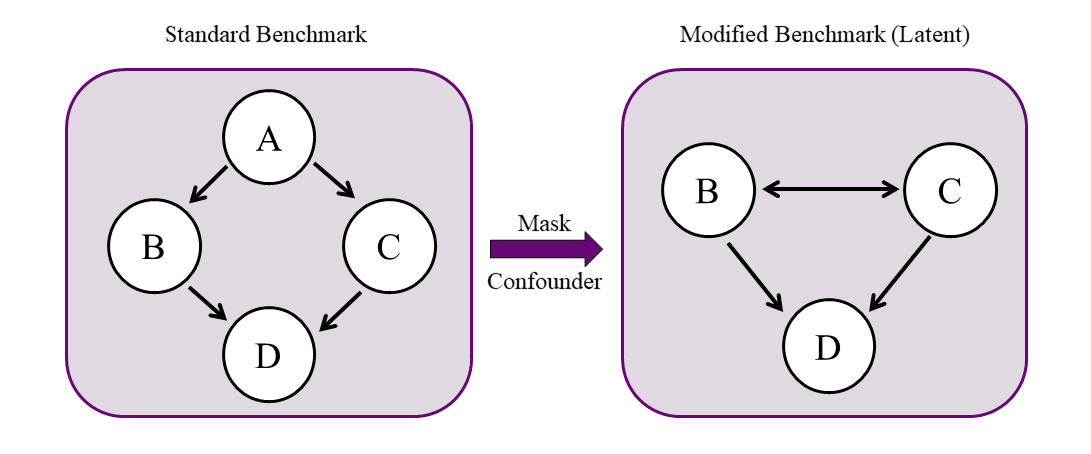}
  \caption{Construction of the Latent benchmark by replacing hidden confounders with bidirected edges.}
  \label{benchmark}
\end{figure}

\subsection{Baselines}

The Direct LLM baseline uses the same base models as Tree-Query (Qwen2.5-7B, DeepSeek-R1-7B, Llama-3-8B, Mistral-8B), but queries each variable pair with a single prompt, without tree-structured decomposition or explicit graph construction.  
Each answer (e.g., ``$X$ causes $Y$'', ``$Y$ causes $X$'', ``no direct causal relation'') is mapped to an edge type, yielding a single categorical decision per pair.  
Since no scores or rankings are produced, we only report SHD for this baseline.

\subsection{Results}

As shown in Table~\ref{tab:tree_query_results}, Tree-Query consistently outperforms Direct LLM querying across all base models and both benchmarks. Tree-Query achieves NDCG between $0.73$ and $0.81$ on Standard and $0.68$–$0.76$ on Latent, while Direct LLM has no NDCG by design. In terms of SHD, Tree-Query reduces structural errors by roughly $20$ edges per graph compared to Direct LLM in both benchmarks, indicating substantially more accurate graph reconstruction even in the more challenging Latent setting.

\begin{table}[t]
\centering
\small
\resizebox{1\linewidth}{!}{
\begin{tabular}{l l | cc | cc}
\toprule
& & \multicolumn{2}{c|}{\textbf{Standard}} & \multicolumn{2}{c}{\textbf{Latent}} \\
\textbf{Method} & \textbf{Base LLM} & \textbf{NDCG}\textsubscript{↑} & \textbf{SHD}\textsubscript{↓} & \textbf{NDCG}\textsubscript{↑} & \textbf{SHD}\textsubscript{↓} \\
\midrule
Direct LLM   & Qwen2.5-7B      & N/A          & 52.3 (2.8) & N/A          & 59.1 (2.6) \\
Tree-Query   & Qwen2.5-7B      & 0.81 (0.02)  & 31.4 (3.1) & 0.76 (0.01)  & 38.7 (1.9) \\
\midrule
Direct LLM   & DeepSeek-R1-7B  & N/A          & 54.8 (3.0) & N/A          & 62.0 (2.7) \\
Tree-Query   & DeepSeek-R1-7B  & 0.79 (0.06)  & 33.7 (3.4) & 0.74 (0.03)  & 41.3 (2.0) \\
\midrule
Direct LLM   & Llama-3-8B      & N/A          & 60.5 (3.4) & N/A          & 67.8 (3.2) \\
Tree-Query   & Llama-3-8B      & 0.75 (0.03)  & 38.9 (2.9) & 0.70 (0.06)  & 46.2 (2.4) \\
\midrule
Direct LLM   & Mistral-8B      & N/A          & 62.1 (3.5) & N/A          & 69.2 (3.3) \\
Tree-Query   & Mistral-8B      & 0.73 (0.03)  & 40.7 (3.4) & 0.68 (0.02)  & 48.5 (2.5) \\
\bottomrule
\end{tabular}
}
\caption{Comparison of Tree-Query and direct LLM causal querying across two benchmarks (values are mean (std) over runs).}
\label{tab:tree_query_results}
\end{table}

\begin{table}[t]
\centering
\small
\resizebox{\linewidth}{!}{
\begin{tabular}{l | c c c c}
\toprule
\textbf{Property} &
\textbf{Tree-Query} &
\textbf{Direct LLM} &
\textbf{PC} &
\textbf{FCI} \\
\midrule
Output Type & DAG(Conf.) & DAG & CPDAG & PAG \\
Detects Confounders & \checkmark & \xmark & \xmark & \checkmark \\
Outputs Confidence & \checkmark & \xmark & \xmark & \xmark \\
Transparent Process & \checkmark & \xmark & \checkmark & \checkmark \\
Data-Free & \checkmark & \checkmark & \xmark & \xmark \\
\bottomrule
\end{tabular}
}
\caption{Comparison across key properties for Tree-Query, Direct LLM, PC, and FCI.}
\end{table}

\section{Case Study: Confounder Screening for the Effect of Diet on Weight}

This section illustrates how Tree-Query performs confounder screening for the causal effect of dietary intervention (\texttt{diet}) on weight change (\texttt{weight}). The model outputs one of four causal relations with an associated confidence score. Starting from a bivariate setting, additional variables---exercise (\texttt{exercise}), family support (\texttt{family support}), and basal metabolic rate (\texttt{BMR})---are introduced to show how confounding and causal confidence evolve. As shown in Figure~\ref{case_demo}, the confounding confidence decreases as relevant covariates enter, while the causal direction (\texttt{diet} $\rightarrow$ \texttt{weight}) remains stable.

\begin{figure}[h]
  \vskip 0.1in
  \begin{center}
    \centerline{\includegraphics[width=0.5\textwidth]{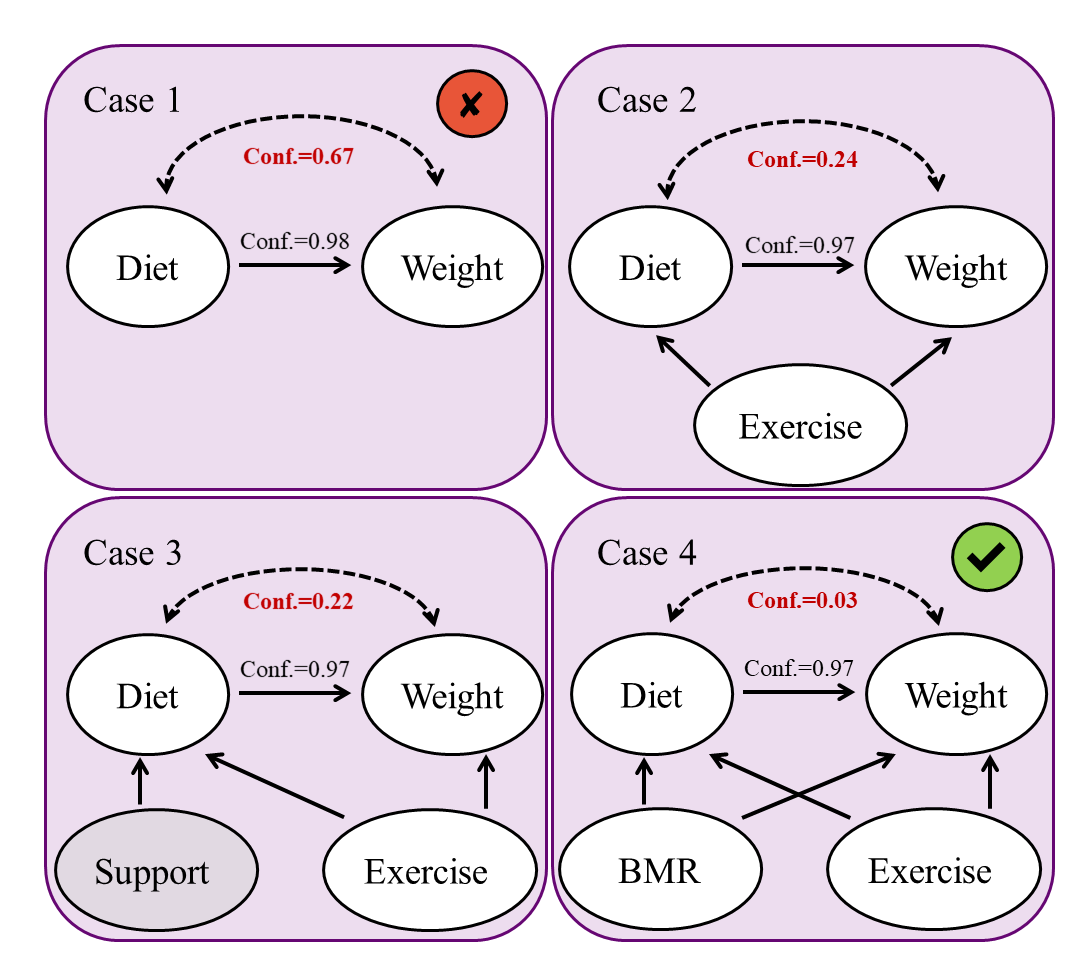}}
    \caption{Confounder screening for the effect of diet on weight using Tree-Query. As additional variables are introduced, the confounding confidence (dashed arrows) decreases while the causal direction (diet $\rightarrow$ weight) remains stable.}
    \label{case_demo}
  \end{center}
\end{figure}

\textbf{Case 1: Only diet and weight.}\\
With $\{\texttt{diet},\texttt{weight}\}$, Tree-Query infers \texttt{diet} $\rightarrow$ \texttt{weight} with high confidence (0.98), but also reports substantial confounding (0.67), indicating a strong causal signal that may still be biased by omitted factors.

\textbf{Case 2: Adding exercise.}\\
With $\{\texttt{diet},\texttt{weight},\texttt{exercise}\}$, the direction \texttt{diet} $\rightarrow$ \texttt{weight} remains confident (0.97), while confounding confidence drops to about 0.24, suggesting that exercise explains a large portion of the previous bias.

\textbf{Case 3: Adding family support.}\\
Extending to the variable set \{\texttt{diet}, \texttt{weight}, \texttt{exercise}, \texttt{family support}\} 
leaves the causal direction unchanged (0.97), with only a small further reduction in confounding (to 0.22). 
Family support is effectively identified as weakly informative for this relation.

\textbf{Case 4: Adding exercise and BMR.}\\
For $\{\texttt{diet},\texttt{weight},\texttt{exercise},\texttt{BMR}\}$, Tree-Query finds clear non-independence (conf.=1.0) and a stable direction \texttt{diet} $\rightarrow$ \texttt{weight} (conf.=0.97), while confounding probability drops to about 0.03. Joint adjustment for exercise and BMR effectively removes remaining confounding pathways.

\textbf{Summary.}\\
Across these stages, Tree-Query separates genuine causal effects from spurious associations: exercise emerges as a key confounder, family support adds little, and incorporating both exercise and BMR yields a stable, high-confidence estimate of the causal impact of diet on weight.

\section{Conclusion}
This work presented Tree-Query, a multi-expert LLM framework that turns pairwise causal discovery into a short, transparent sequence of queries with robustness-aware confidence scores. Coupled with the Adversarial Confidence Estimator and a modular Multi-Expert System, Tree-Query provides interpretable local judgments, theoretical guarantees on identifiability, and improved structural accuracy on data-free benchmarks with and without latent confounders. A qualitative case study on diet and weight further illustrates how the framework screens confounders and converges to stable, high-confidence causal conclusions. 

While our approach operates in a data-free setting, we do not intend it as a replacement for data-driven causal discovery. Instead, we view Tree-Query as a principled, interpretable source of causal priors for practitioners, helping to shape hypotheses, guide experimental design, and warm-start downstream structure-learning methods that operate on observational or interventional data. Exploring such hybrid pipelines that combine LLM-derived priors with statistical evidence is an important direction for future work.

\section*{Impact Statement}

While the Tree-Query framework is theoretically grounded, its motivation extends beyond algorithmic optimization. The framework is designed to support the development of an interpretable causal reasoning system that enables researchers—particularly those without extensive backgrounds in causal inference—to understand, visualize, and evaluate causal hypotheses intuitively.

In many experimental sciences, verifying causal relationships through controlled experiments is costly and often infeasible. Tree-Query offers a pre-screening and reasoning mechanism that assists experimental researchers in identifying the most promising causal directions, thereby reducing redundant experimental efforts and associated costs.

By abstracting causal identification into a transparent tree-like structure, Tree-Query lowers the entry barrier for causal reasoning and facilitates broader participation across the scientific community.

The complete Tree-Query implementation has been open-sourced to encourage transparency and community-driven validation. The repository is available at \url{https://anonymous.4open.science/r/Repo-9B3E-4F96/}
, providing researchers with practical tools to test and refine their own causal hypotheses. This open causal ecosystem aims to promote the co-evolution of theoretical advancement and empirical verification.

\bibliography{example_paper}
\bibliographystyle{icml2026}

\newpage
\appendix
\onecolumn

\section{Appendix: DecisionRule Algorithm}
\label{DecisionRule1}

\begin{algorithm}[H]
\caption{DecisionRule / Resolve-Causal-Relation}
\label{alg:decision_rule}
\begin{algorithmic}[1]
  \STATE \textbf{Input:} 
         backdoor flag $\textit{have\_backdoor}$, backdoor confidence $c_{\mathrm{bd}}$, 
         result list $\textit{results}$, threshold $\tau$
  \STATE \textbf{Output:} final relation $\hat{R}$ and confidence $\hat{c}$
  \STATE $\textit{backdoor\_confident} \gets (c_{\mathrm{bd}} \ge \tau)$
  \IF{$\textit{backdoor\_confident}$}
    \STATE $\textit{branches} \gets 
      \begin{cases}
        \{\text{``after\_block''}\}, & \text{if } \textit{have\_backdoor} \\
        \{\text{``no\_backdoor''}\}, & \text{otherwise}
      \end{cases}$
  \ELSE
    \STATE $\textit{branches} \gets \{\text{``after\_block''}, \text{``no\_backdoor''}\}$
  \ENDIF

  \STATE $\mathcal{R} \gets \emptyset$

  \FOR{each $b \in \textit{branches}$}
    \STATE $\textit{B} \gets \{ r \in \textit{results} : r.\texttt{branch\_name} = b \}$
    \STATE $\textit{Ind} \gets \{ r \in \textit{B} : r.\texttt{relation} \in \{\text{independent}, \text{not independent}\} \}$
    \STATE $\textit{Lat} \gets \{ r \in \textit{B} : r.\texttt{relation} \in \{\leftrightarrow, \text{not } \leftrightarrow\} \}$
    \STATE $\textit{Dir} \gets \{ r \in \textit{B} : r.\texttt{relation} \in \{\rightarrow, \leftarrow, \text{not } \rightarrow, \text{not } \leftarrow\} \}$

    \IF{$|\textit{Ind}| > 0$}
      \STATE $r_{\mathrm{ind}} \gets \textsc{First}(\textit{Ind})$; 
             $\mathcal{R} \gets \mathcal{R} \cup \{r_{\mathrm{ind}}\}$
      \IF{$r_{\mathrm{ind}}.\texttt{relation} = \text{independent}$ 
          \AND $r_{\mathrm{ind}}.\texttt{confidence} \ge \tau$}
        \STATE \textbf{continue}
      \ENDIF
    \ENDIF

    \IF{$|\textit{Lat}| > 0$}
      \STATE $r_{\mathrm{lat}} \gets \textsc{First}(\textit{Lat})$; 
             $\mathcal{R} \gets \mathcal{R} \cup \{r_{\mathrm{lat}}\}$
      \IF{$r_{\mathrm{lat}}.\texttt{relation} = \leftrightarrow$ 
          \AND $r_{\mathrm{lat}}.\texttt{confidence} \ge \tau$}
        \STATE \textbf{continue}
      \ENDIF
    \ENDIF

    \STATE $\textit{PosDir} \gets \{ r \in \textit{Dir} : r.\texttt{relation} \in \{\rightarrow, \leftarrow\} \}$
    \IF{$|\textit{PosDir}| > 0$}
      \STATE $\mathcal{R} \gets \mathcal{R} \cup \textit{PosDir}$
    \ELSE
      \STATE $r_{\text{not}\rightarrow} \gets$ element of $\textit{Dir}$ with relation ``not $\rightarrow$'' (if any)
      \STATE $r_{\text{not}\leftarrow} \gets$ element of $\textit{Dir}$ with relation ``not $\leftarrow$'' (if any)
      \IF{$r_{\text{not}\rightarrow}$ and $r_{\text{not}\leftarrow}$ exist}
        \IF{$r_{\text{not}\rightarrow}.\texttt{confidence} 
            \le r_{\text{not}\leftarrow}.\texttt{confidence}$}
          \STATE construct $r^\ast$ with 
                 $r^\ast.\texttt{relation} \gets \rightarrow$, 
                 $r^\ast.\texttt{confidence} \gets 
                   1 - r_{\text{not}\rightarrow}.\texttt{confidence}$
        \ELSE
          \STATE construct $r^\ast$ with 
                 $r^\ast.\texttt{relation} \gets \leftarrow$, 
                 $r^\ast.\texttt{confidence} \gets 
                   1 - r_{\text{not}\leftarrow}.\texttt{confidence}$
        \ENDIF
        \STATE $\mathcal{R} \gets \mathcal{R} \cup \{r^\ast\}$
      \ENDIF
    \ENDIF
  \ENDFOR

  \STATE $\textit{VALID} \gets \{\text{independent}, \leftrightarrow, \rightarrow, \leftarrow\}$
  \STATE $\textit{ValidResults} \gets \{ r \in \mathcal{R} : r.\texttt{relation} \in \textit{VALID} \}$

  \IF{$|\textit{ValidResults}| = 0$}
    \STATE \textbf{return} $(\hat{R} = \text{unknown}, \hat{c} = 0.0)$
  \ENDIF

  \STATE $r^{\star} \gets \arg\max_{r \in \textit{ValidResults}} r.\texttt{confidence}$
  \STATE \textbf{return} $(\hat{R} = r^{\star}.\texttt{relation}, \hat{c} = r^{\star}.\texttt{confidence})$

\end{algorithmic}
\end{algorithm}

\section{Appendix: Multi-Expert System and Adversarial Confidence Estimator}
\label{appendix:mes_ace}

\begin{algorithm}[H]
\caption{Multi-Expert System with Adversarial Confidence Estimation}
\label{alg:mes_ace_combined}
\begin{algorithmic}[1]
  \STATE \textbf{Input:} query type $Q$, variables $(X_1, X_2)$, variable set $V$, 
         total experts $K$, selected experts $m$ ($m \le K$), 
         sample size $N$, personas $\mathcal{P}=\{P_1,\dots,P_n\}$, 
         persona weights $\lambda_1,\dots,\lambda_n$ (default: $\lambda_j = 1/n$)
  \STATE \textbf{Output:} binary label $\hat{y} \in \{0,1\}$, confidence $c \in [0,1]$, expert log $\mathcal{R}$
  \STATE
  \STATE \textbf{// PART A: Multi-Expert System (MES)}
  \STATE $\mathcal{E}_{\textnormal{base}} \gets \textsc{RoutingRules}(Q)$
  \STATE $\mathcal{E} \gets \textsc{ClinicSelect}(Q, X_1, X_2, V, \mathcal{E}_{\textnormal{base}}, m)$
  \STATE $\mathcal{R} \gets \emptyset$
  \FOR{each expert $e \in \mathcal{E}$}
    \STATE $\ell \gets \textsc{QueryExpert}(e, Q, X_1, X_2, V)$
    \STATE $\mathcal{R} \gets \mathcal{R} \cup \{(e,\ell)\}$
  \ENDFOR
  \STATE $\hat{y} \gets \textsc{MajorityVote}(\mathcal{R})$
  \STATE
  \STATE \textbf{// PART B: Adversarial Confidence Estimator (ACE)}
  \STATE \textbf{// Step B1: Construct natural-language query}
  \STATE $x_{\textnormal{q}} \gets \textsc{FormatQuery}(Q, X_1, X_2, V)$
  \STATE Let $MES^\ast$ denote calling $MES$ with experts fixed to $\mathcal{R}$
  \STATE \textbf{// Step B2: Collect original answers ($N$ runs)}
  \FOR{$i=1$ \textbf{to} $N$}
    \STATE $x_{\textnormal{ans}}^{(i)} \gets MES^\ast(x_{\textnormal{q}})$
    \STATE $x_{\textnormal{concl}}^{(i)} \gets \textsc{Extract}(x_{\textnormal{ans}}^{(i)})$
  \ENDFOR
  \STATE \textbf{// Step B3: Compute baseline consensus}
  \STATE $N_{\textnormal{yes}} \gets |\{i : x_{\textnormal{concl}}^{(i)} = \text{``Yes''}\}|$; $N_{\textnormal{no}} \gets |\{i : x_{\textnormal{concl}}^{(i)} = \text{``No''}\}|$
  \STATE $p_0^{\textnormal{raw}} \gets \frac{\max(N_{\textnormal{yes}}, N_{\textnormal{no}})}{N}$
  \STATE $p_0 \gets 2p_0^{\textnormal{raw}} - 1$
  \STATE $y^\ast \gets \arg\max_{y \in \{\text{``Yes''}, \text{``No''}\}} |\{i : x_{\textnormal{concl}}^{(i)} = y\}|$
  \STATE \textbf{// Step B4: Adversarial re-query for each persona}
  \FOR{$j=1$ \textbf{to} $n$}
    \FOR{$i=1$ \textbf{to} $N$}
      \STATE $x_{\textnormal{adv},j}^{(i)} \gets \textsc{AdvGen}(P_j, x_{\textnormal{q}}, x_{\textnormal{ans}}^{(i)})$
      \STATE ${x'}_{\textnormal{ans},j}^{(i)} \gets MES^\ast(x_{\textnormal{q}}, x_{\textnormal{ans}}^{(i)}, x_{\textnormal{adv},j}^{(i)})$
      \STATE ${x'}_{\textnormal{concl},j}^{(i)} \gets \textsc{Extract}({x'}_{\textnormal{ans},j}^{(i)})$
    \ENDFOR
    \STATE $p_j \gets \frac{1}{N}\sum_{i=1}^N \mathbbm{1}({x'}_{\textnormal{concl},j}^{(i)}=y^\ast)$
  \ENDFOR
  \STATE \textbf{// Step B5: Confidence aggregation}
  \STATE $c \gets p_0 \Bigl(1 - \sum_{j=1}^n \lambda_j \frac{|p_j-p_0^{\textnormal{raw}}|}{p_0^{\textnormal{raw}}}\Bigr)$
  \STATE
  \STATE \textbf{return} $(\hat{y}, c, \mathcal{R})$
\end{algorithmic}
\end{algorithm}

\section{Appendix: Expert types and prompt templates}
\label{expert_prompt}

\begin{table*}[h]
\centering
\caption{Example expert types used in the Multi-Expert System.  
The list is not exhaustive; it illustrates several possible causal reasoning perspectives.}
\label{tab:experts}
\begin{tabular}{ll}
\toprule
\textbf{Expert Name} & \textbf{Description} \\
\midrule
Structural Causal Expert & Focuses on SCM semantics and graph-level causal relations. \\
Statistical Independence Expert & Emphasizes conditional independence and probabilistic constraints. \\
Intervention Reasoning Expert & Uses do-interventions and counterfactual perturbations. \\
Backdoor Criterion Expert & Specializes in identifying valid adjustment sets. \\
Causal Direction Expert & Evaluates directional relations and v-structure patterns. \\
Path Analysis Expert & Examines all directed and undirected paths between variables. \\
Graph Transformation Expert & Applies graph surgery, removal, and modification heuristics. \\
\bottomrule
\end{tabular}
\end{table*}

This appendix provides the complete textual templates used by each expert in the Multi-Expert System for Causal Judgment. For every expert listed in Table~\ref{tab:experts}, we include:

\begin{itemize}
    \item the full prompt template used to instruct the expert,
    \item formatting conventions for variables and causal structures,
    \item an illustrative input example,
    \item and example outputs for both possible labels (``Yes'' or ``No'').
\end{itemize}

These templates ensure that experts---although operating under heterogeneous causal perspectives---produce comparable binary decisions suitable for aggregation within the Multi-Expert System.

\subsection{General Formatting Rules}

All experts receive a base query generated by the function \texttt{BuildExpertPrompt(Q, X1, X2, V, ei)}, which creates a specialized prompt based on the input variables and the expert type:

\begin{verbatim}
BuildExpertPrompt(Q, X1, X2, V, ei)
\end{verbatim}

Where:

\begin{itemize}
    \item \texttt{Q} is the causal task-specific question involving the variables \texttt{X1} and \texttt{X2}.
    \item \texttt{X1} and \texttt{X2} are the two variables under analysis.
    \item \texttt{V} is a list of all relevant variables in the causal context.
    \item \texttt{ei} specifies the expert type (e.g., Graph-Theory Expert, Statistical Expert, etc.).
\end{itemize}

The resulting prompt follows the structure:

\begin{verbatim}
In causal inference, consider the following variables: {V}
{Q}
Let us think step by step, and then output directly Yes or No.
\end{verbatim}

All experts must output \textbf{only}:

\begin{verbatim}
Yes
\end{verbatim}

or

\begin{verbatim}
No
\end{verbatim}

No explanation, probability, or justification is included in the expert output.

Below are the complete templates generated using \texttt{BuildExpertPrompt(Q, X1, X2, V, ei)} for each expert.

\subsubsection{Graph-Theory Expert}

\textbf{Prompt Template}

\begin{verbatim}
As a causal graph theory expert, follow the framework below:

Specialization: causal graph structure, d-separation, path blocking.
Strengths: path analysis, cycle detection, d-separation evaluation.
Reasoning style: structured graph traversal.
Output requirement: binary decision based solely on graph structure.

Steps:
1. Construct the causal graph and identify all possible paths.
2. Apply d-separation to analyze path blocking.
3. Examine backdoor, frontdoor, and confounding paths.
4. Produce a clear Yes/No judgment based on graph structure.

Please follow graph-theoretic principles strictly and output only Yes or No.

{base_prompt}
\end{verbatim}

---

\subsubsection{Statistical Expert}

\textbf{Prompt Template}

\begin{verbatim}
As a statistical inference expert, follow the framework below:

Specialization: statistical testing, independence assessment.
Strengths: independence tests, correlation analysis, confounder detection.
Reasoning style: probabilistic-statistical reasoning.
Output requirement: binary decision based on statistical evidence.

Steps:
1. Assess statistical correlation between variables.
2. Consider conditional independence and confounding factors.
3. Analyze statistical significance and robustness.
4. Produce a clear Yes/No judgment based on statistical reasoning.

Please apply statistical principles rigorously and output only Yes or No.

{base_prompt}
\end{verbatim}

---

\subsubsection{Domain-Knowledge Expert}

\textbf{Prompt Template}

\begin{verbatim}
As a domain knowledge expert, follow the reasoning framework below:

Specialization: real-world scientific and commonsense causal reasoning.
Strengths: mechanism analysis, temporal logic, physical/social constraints.
Reasoning style: evidence-based inductive reasoning.
Output requirement: binary judgment based on domain knowledge.

Steps:
1. Apply relevant scientific knowledge and commonsense reasoning.
2. Evaluate physical/biological/social mechanism plausibility.
3. Assess temporal order and real-world feasibility.
4. Produce a clear Yes/No judgment grounded in domain knowledge.

Please integrate real-world understanding and output only Yes or No.

{base_prompt}
\end{verbatim}

---

\subsubsection{Counterfactual Expert}

\textbf{Prompt Template}

\begin{verbatim}
As a counterfactual inference expert, follow the framework below:

Specialization: intervention analysis, potential outcomes, do-operator.
Strengths: counterfactual simulation and causal effect evaluation.
Reasoning style: counterfactual thought experiments.
Output requirement: binary judgment based on interventional reasoning.

Steps:
1. Construct the intervention scenario (do-operation).
2. Compare actual outcome with counterfactual outcome.
3. Analyze the distribution of potential outcomes.
4. Produce a clear Yes/No judgment based on intervention effects.

Please use counterfactual reasoning strictly and output only Yes or No.

{base_prompt}
\end{verbatim}

---

\subsubsection{Temporal-Dynamics Expert}

\textbf{Prompt Template}

\begin{verbatim}
As a temporal dynamics expert, follow the framework below:

Specialization: temporal order, dynamic processes.
Strengths: time-series reasoning, lag effects, dynamic causal structures.
Reasoning style: temporal causal analysis.
Output requirement: binary judgment based on temporal constraints.

Steps:
1. Check strict temporal ordering between cause and effect.
2. Analyze lagged effects and dynamic processes.
3. Evaluate time-series causal structures.
4. Produce a Yes/No judgment based on temporal logic.

Please focus on the temporal dimension and output only Yes or No.

{base_prompt}
\end{verbatim}

---

\subsubsection{Mechanism-Modeling Expert}

\textbf{Prompt Template}

\begin{verbatim}
As a mechanism modeling expert, follow the framework below:

Specialization: interpretable causal mechanisms and mediation.
Strengths: mediator analysis, mechanism consistency, functional relations.
Reasoning style: mechanism decomposition.
Output requirement: binary judgment based on mechanism completeness.

Steps:
1. Identify possible mediators and intermediate mechanisms.
2. Examine completeness of the causal chain.
3. Evaluate the plausibility of mechanism pathways.
4. Produce a clear Yes/No judgment based on mechanism reasoning.

Please focus on mechanism-level analysis and output only Yes or No.

{base_prompt}
\end{verbatim}

---

\subsubsection{Robustness-Analysis Expert}

\textbf{Prompt Template}

\begin{verbatim}
As a robustness analysis expert, follow the framework below:

Specialization: robustness and sensitivity of causal claims.
Strengths: sensitivity tests, boundary-case evaluation, robustness validation.
Reasoning style: multi-scenario comparison.
Output requirement: binary judgment based on robustness assessment.

Steps:
1. Test result stability under varying assumptions.
2. Perform sensitivity and edge-case analysis.
3. Evaluate robustness of causal conclusions.
4. Produce a Yes/No judgment based on robustness.

Please evaluate the reliability of the conclusion and output only Yes or No.

{base_prompt}
\end{verbatim}

---

\subsection{Example Input Format}

\begin{verbatim}
Variables: ["Ice Cream Sales", "Drowning Incidents", "Temperature"]
Task: Determine whether a backdoor path exists between
       Ice Cream Sales and Drowning Incidents.
\end{verbatim}

The system constructs:

\begin{itemize}
    \item a base prompt,
    \item an expert-specific prompt for each expert,
    \item and each expert outputs Yes/No.
\end{itemize}

---

\subsection{Example Output Format}

Each expert returns:

\begin{verbatim}
Yes
\end{verbatim}

or

\begin{verbatim}
No
\end{verbatim}

When majority voting is used, the Multi-Expert System aggregates \textbf{only labels} (not probabilities).

\section{Appendix: Proof of Theorem~\ref{thm:asym_identifiability}}
\label{appendix:proof_asym_identifiability}

\begin{proof}

\textbf{Step 1: Reliability of each decision query.}
Each Tree-Query decision query (\texttt{backdoor\_path}, \texttt{independence}, \texttt{latent\_confounder}, or \texttt{causal\_direction}) aggregates $m$ independent expert votes by majority rule.  
Let $X_i$ denote the indicator variable that expert $i$ outputs the correct label, with $\mathbb{E}[X_i] \ge 1 - \alpha$.  
The probability that the aggregated majority decision is wrong is
\[
P(\textnormal{step error}) = P\!\left(\sum_{i=1}^m X_i < \frac{m}{2}\right)
   = P\!\left(\sum_{i=1}^m X_i - m(1-\alpha) < -m(0.5-\alpha)\right).
\]
Applying Hoeffding’s inequality yields
\[
P(\textnormal{step error}) \le \exp[-2m(0.5-\alpha)^2].
\]
Denote this bound as $\varepsilon_k$.

\textbf{Step 2: Overall Tree-Query reliability.}
Tree-Query consists of $M$ deterministic decision steps.  
By the union bound, the probability that any of them errs satisfies
\[
P(\textnormal{overall error}) \le M\,\varepsilon_m
    = M\,\exp[-2k(0.5-\alpha)^2].
\]
Hence the probability that the entire Tree-Query outputs the correct causal relation is
\[
P_{\textnormal{correct}} \ge 1 - M\,\exp[-2m(0.5-\alpha)^2].
\]

\textbf{Step 3: Correctness for each causal relation type.}
For any variable pair $(X_1,X_2)$, let the true relation be
\[
R \in \{\,X_1 \perp X_2,\; X_1 \leftrightarrow X_2,\;
              X_1 \rightarrow X_2,\; X_2 \rightarrow X_1\,\}.
\]
Tree-Query identifies $R$ through a deterministic combination of these querys:
\begin{itemize}
    \item $X_1 \perp X_2$: correctness requires the independence query to be correct,
    giving $P_{\textnormal{correct}}(R) \ge 1-\varepsilon_k$.
    \item $X_1 \leftrightarrow X_2$: requires independence and latent-confounder querys
    both correct, giving $P_{\textnormal{correct}}(R) \ge 1-2\varepsilon_k$.
    \item $X_1 \rightarrow X_2$ or $X_2 \rightarrow X_1$: requires backdoor,
    independence, latent, and direction querys all correct, yielding
    $P_{\textnormal{correct}}(R) \ge 1-4\varepsilon_k$.
\end{itemize}
Since $m$ upper-bounds the total number of potential failure points, we have the general bound
\[
P_{\textnormal{correct}}(R) \ge 1 - M\,\varepsilon_m
   = 1 - M\,e^{-2m(0.5-\alpha)^2}.
\]

\textbf{Step 4: Asymptotic identifiability.}
Because $\alpha < 0.5$, we have $(0.5-\alpha) > 0$, hence
\[
\lim_{m\to\infty} e^{-2m(0.5-\alpha)^2} = 0
   \quad\Rightarrow\quad
   \lim_{m\to\infty} P_{\textnormal{correct}}(R) = 1.
\]
Similarly, for any finite $m$, as $\alpha \to 0$ we also obtain
$P_{\textnormal{correct}}(R) \to 1$.
Therefore, Tree-Query is \textbf{asymptotically identifiable}: each causal relation type
($\rightarrow, \leftarrow, \leftrightarrow, \perp$) can be correctly detected with probability
approaching one as the number of experts increases or as individual expert reliability improves.
\end{proof}

\section{Appendix: Proof of Proposition~\ref{thm1}.}
\label{appendix:Prop1}

\begin{proof}
Since Tree-Query makes independent judgments for each potential edge:
\[
E_{\textnormal{TQ}} = \sum_{i=1}^e P(\textnormal{edge } i \textnormal{ correct}) = e \cdot P_{\textnormal{pair}}
\]
Substituting the lower bound from Theorem 1 yields the result.
\end{proof}

\section{Appendix: Proof of Proposition~\ref{prop2}.}
\label{appendix:Prop2}
\begin{proof}
From Theorem~\ref{thm:asym_identifiability}:
\[
P_{\textnormal{pair}} \geq 1 - M \cdot \exp\left(-2m(0.5-\alpha)^2\right) \geq P_{\textnormal{target}}
\]
Rearranging gives the stated inequality.
\end{proof}

\section{Appendix: Prompt templates for ACE}
\label{app:ace_prompts}

For the space of all finite text sequences $\mathcal{X}$, let $x_{\textnormal{q}} \in \mathcal{X}$ be the prompt question. The LLM is guided to output both a binary answer and its reasoning:

\begin{center}
  \fbox{
    \begin{minipage}{0.9\linewidth}
      Question: $x_{\textnormal{q}}$ \\
      Think through this step by step, then answer Yes or No with very detailed reasoning like \texttt{"Answer: 'Yes'/'No'\textbackslash n\textbackslash nReasoning:\textbackslash n'Your reasoning here'"}.

      \noindent\makebox[\linewidth]{\dotfill}
      \vspace{-0.5em}
      
      $\left.\begin{minipage}[c]{0.85\linewidth}
        Answer: $x_{\textnormal{concl}}$ \\
        Reasoning: Detailed reasoning steps...
      \end{minipage}\right\} \; x_{\textnormal{ans}}$
    \end{minipage}
  }
\end{center}

To generate adversarial agents $x_{\textnormal{adv}} \in \mathcal{X}$ that support the opposite conclusion, the following template is used:
\[
\{x_{\textnormal{adv. concl}}\} = \{\textnormal{``Yes''}, \textnormal{``No''}\} \backslash \{x_{\textnormal{concl}}\}.
\]

\begin{center}
  \fbox{
    \begin{minipage}{0.9\linewidth}
      Original Question: $x_{\textnormal{q}}$

      Original Answer: $x_{\textnormal{ans}}$

      Original Conclusion: $x_{\textnormal{concl}}$

      Your task: Argue that the correct answer should be ``$x_{\textnormal{adv. concl}}$''.

      \noindent\makebox[\linewidth]{\dotfill}
      \vspace{-0.5em}
      
      $\left.\begin{minipage}[c]{0.85\linewidth}
        Adversarial arguments...
      \end{minipage}\right\} \; x_{\textnormal{adv}}$
    \end{minipage}
  }
\end{center}

The original answer and adversarial arguments are then fed back to the LLM for reassessment:

\begin{center}
  \fbox{
    \begin{minipage}{0.9\linewidth}
      Question: $x_{\textnormal{q}}$

      Your previous answer was: $x_{\textnormal{ans}}$

      Consider the following statement: $x_{\textnormal{adv}}$

      Now consider the question again carefully. What is your opinion now? Provide a clear Yes or No with very detailed reasoning like \texttt{"Answer: 'Yes'/'No'\textbackslash n\textbackslash nReasoning:\textbackslash n'Your reasoning here'"}.

      \noindent\makebox[\linewidth]{\dotfill}
      \vspace{-0.5em}
      
      $\left.\begin{minipage}[c]{0.85\linewidth}
        Answer: $x'_{\textnormal{concl}}$ \\
        Reasoning: Detailed reasoning steps...
      \end{minipage}\right\} \; x'_{\textnormal{ans}}$
    \end{minipage}
  }
\end{center}

In the main experiments, $n = 3$ adversarial personas are instantiated:
\begin{enumerate}
    \item \textbf{Contrarian:} systematically refutes the original reasoning through counterarguments and alternative interpretations.
    \item \textbf{Deceiver:} constructs opposing arguments with citation-based rhetorical strategies and empirical references, including fabricated resources if necessary.
    \item \textbf{Hater:} challenges answer credibility through affectively charged language and conviction-based discourse.
\end{enumerate}

\section{Appendix: Proof of Proposition~\ref{prop:confidence_robustness}.}
\label{proof:confi}

Recall that $y^*$ is the majority conclusion of the original answers. The empirical distribution of the original answers, $\hat{\mu}_{\textnormal{ans}}$, is defined by $\hat{\mu}_{\textnormal{ans}}(y^*) = p_0^{\textnormal{raw}}$ and $\hat{\mu}_{\textnormal{ans}}(\bar{y}^*) = 1 - p_0^{\textnormal{raw}}$, where $\bar{y}^*$ denotes the complement of $y^*$ in $\mathcal{Y}$. Similarly, for the $j$-th adversarial agent, the empirical distribution $\hat{\mu}_{\textnormal{adv}, j}$ is given by $\hat{\mu}_{\textnormal{adv}, j}(y^*) = p_j$ and $\hat{\mu}_{\textnormal{adv}, j}(\bar{y}^*) = 1 - p_j$, as defined in Eq.~(\ref{consistency_prob}).

The Wasserstein-1 distance between two probability measures $\mu$ and $\nu$ on the metric space $(\mathcal{Y}, d)$ is defined as:
\begin{equation}
W_1(\mu, \nu) = \inf_{\gamma \in \Gamma(\mu, \nu)} \mathbb{E}_{(u, v) \sim \gamma} [d(u, v)],
\end{equation}
where $\Gamma(\mu, \nu)$ is the set of all joint distributions (couplings) on $\mathcal{Y} \times \mathcal{Y}$ with marginals $\mu$ and $\nu$. For the discrete metric $d(u, v) = \mathbbm{1}(u \neq v)$, the cost function is non-zero only when $u \neq v$.

For the binary space $\mathcal{Y} = \{y^*, \bar{y}^*\}$, any coupling $\gamma$ is determined by $\gamma(y^*, y^*)$, since the marginal constraints imply:
\begin{align*}
\gamma(y^*, \bar{y}^*) &= \mu(y^*) - \gamma(y^*, y^*) \\
\gamma(\bar{y}^*, y^*) &= \nu(y^*) - \gamma(y^*, y^*)
\end{align*}
The objective function to minimize is:
\begin{align*}
\mathbb{E}[d(u, v)] &= \gamma(y^*, \bar{y}^*) + \gamma(\bar{y}^*, y^*) \\
&= (\mu(y^*) + \nu(y^*)) - 2\gamma(y^*, y^*).
\end{align*}
To minimize this quantity, we must maximize $\gamma(y^*, y^*)$ subject to $\gamma(y^*, y^*) \leq \min(\mu(y^*), \nu(y^*))$. Setting $\gamma(y^*, y^*) = \min(\mu(y^*), \nu(y^*))$, we obtain:
\begin{align*}
W_1(\mu, \nu) &= \mu(y^*) + \nu(y^*) - 2\min(\mu(y^*), \nu(y^*)) \\
&= |\mu(y^*) - \nu(y^*)|.
\end{align*}
Substituting the empirical measures $\hat{\mu}_{\textnormal{ans}}$ and $\hat{\mu}_{\textnormal{adv}, j}$:
\begin{equation}
W_1(\hat{\mu}_{\textnormal{ans}}, \hat{\mu}_{\textnormal{adv}, j}) = |\hat{\mu}_{\textnormal{ans}}(y^*) - \hat{\mu}_{\textnormal{adv}, j}(y^*)| = |p_0^{\textnormal{raw}} - p_j|.
\end{equation}
This demonstrates that the term $|p_j - p_0^{\textnormal{raw}}|$ precisely quantifies the distributional shift between the original and adversarial answer distributions under the Wasserstein-1 metric.

Substituting this into Eq.~(\ref{eq:confidence_score}), the confidence score becomes:
\begin{equation}
c = p_0 \cdot \left(1 - \sum_{j=1}^{n} \lambda_j \cdot \frac{W_1(\hat{\mu}_{\textnormal{ans}}, \hat{\mu}_{\textnormal{adv}, j})}{p_0^{\textnormal{raw}}}\right).
\end{equation}
The term $\sum_{j=1}^{n} \lambda_j W_1(\hat{\mu}_{\textnormal{ans}}, \hat{\mu}_{\textnormal{adv}, j})$ represents the expected distributional shift induced by the adversarial agents. Therefore, the confidence score is a monotonically decreasing function of this shift, effectively measuring the extent to which the original answer distribution is invariant (i.e., not affected) by the adversarial agents.


\end{document}